\documentclass[runningheads]{llncs}
\usepackage[T1]{fontenc}
\usepackage{amsmath, amssymb}
\usepackage{booktabs}
\usepackage{enumitem}
\usepackage{graphicx}
\usepackage{dsfont}
\usepackage{xcolor}
\begin{document}
\title{Consensus in Motion: A Case of \\ Dynamic Rationality of Sequential Learning \\ in Probability Aggregation}
\titlerunning{Consensus in Motion: A Case of Dynamic Rationality of Sequential Learning}
%
\author{Polina Gordienko$^1$ \and
Christoph Jansen$^2$ \and
Thomas Augustin$^1$ \and
Martin Rechenauer$^3$
}
\authorrunning{P.~Gordienko et al.}
%

\institute{
$^1$Department of Statistics, Ludwig-Maximilians-Universität München\ $^2$School of Computing \& Communication, Lancaster University Leipzig \ 
$^3$Munich Center for Mathematical Philosophy, Ludwig-Maximilians-Universität München\\ 
Correspondence: \email{Polina.Gordienko@stat.uni-muenchen.de}}
\maketitle              
\begin{abstract}
We propose a framework for probability aggregation based on propositional probability logic. Unlike conventional judgment aggregation, which focuses on static rationality, our model addresses dynamic rationality by ensuring that collective beliefs update consistently with new information. We show that any consensus-compatible and independent aggregation rule on a non-nested agenda is necessarily linear. Furthermore, we provide sufficient conditions for a fair learning process, where individuals initially agree on a specified subset of propositions known as the common ground, and new information is restricted to this shared foundation. This guarantees that updating individual judgments via Bayesian conditioning—whether performed before or after aggregation—yields the same collective belief. A distinctive feature of our framework is its treatment of sequential decision-making, which allows new information to be incorporated progressively through multiple stages while maintaining the established common ground. We illustrate our findings with a running example in a political scenario concerning healthcare and immigration policies.

\end{abstract}
\section{Introduction}

Traditional aggregation methods, such as majority voting, are notorious for producing collective judgments that are logically inconsistent and for failing to meet minimal rationality constraints. This shortcoming lies at the heart of the impossibility theorems that have emerged in judgment aggregation theory over the past two decades (e.g., \cite{paperdietrich07,paper2002,paperlistdietrich07,paperlist12,papermongin08,paperpaulyhees,papernehringpuppe07,DIETRICH2015463,nehring2010,JANSEN201849}). Unlike social choice—where preferences involve ranking alternatives -- the aggregation of judgments concerns interconnected issues that capture acts of assent or dissent. Despite numerous attempts to weaken constraints on aggregation functions in order to achieve logical consistency \cite{paperpivato09,paperpivato14}, prior research has focused mainly on static rationality.

In contrast, this paper advocates a shift toward a \textit{dynamic} theory of collective decision-making in which both individual and collective beliefs update rationally in response to new information (see \cite{paperbayes}). Motivated by the first impossibility result for dynamically rational judgment aggregation \cite{dietrichlist24}, our main contribution is a possibility result for dynamically rational collective decision-making on logically interconnected propositions. We introduce a framework of probability aggregation based on propositional probability logic that generalizes traditional approaches from both judgment and preference aggregation. Building on prior work in probability aggregation \cite{madansky,paperaczelwagner,paperaczelwagner2,genest2,papermcconway,paperlistdietrichprob,DietrichListGeneralizedAgendas}, our first theorem establishes that any \textit{consensus-compatible} and \textit{independent} aggregation rule on a \textit{non-nested} agenda must be \textit{linear}. Furthermore, we show that dynamically rational linear aggregation is attainable when the domain of the aggregation function is restricted to profiles in which all group members share the same probabilistic judgments for a specified subset of propositions—what we term the \textit{common ground} (here understood in line with a broader notion of consensus)\footnote{Throughout this paper, “consensus” is used in two senses: (i) as a broad notion of mutual agreement—here, synonymous with our technical term “common ground” and (ii) as “consensus compatibility,” the requirement that if all individuals assign probability 1 to a proposition, then the collective judgment must also assign 1.}. This common ground represents the set of propositions on which all individuals already agree, and by ensuring that new information is limited to this shared foundation—a process we call \textit{fair learning}—all members benefit equally, preventing dynamically irrational outcomes arising from unequal starting beliefs.  Importantly, we demonstrate that fair learning not only achieves dynamic rationality but also preserves the common ground across \textit{sequential} updates. 

The paper is organized as follows. In Section~\ref{ppl} and Section~\ref{fpa} we develop our framework of probability aggregation based on propositional probability logic. Section~\ref{example} introduces a running example from the context of political decision-making, which we use to illustrate our framework throughout the paper. Section~\ref{probagr} is divided into two parts: Section~\ref{static} presents our characterization result for the linear aggregation rule, and Section~\ref{dynamic} introduces a Bayesian updating process that demonstrates dynamic rationality under a restricted domain. Section~\ref{crem} concludes.

\section{The model of probability aggregation} \label{mpa}
\subsection{Propositional probability logic} \label{ppl}
Let $\mathcal{L}$ be the set of all logically non-equivalent\footnote{Logical equivalence (i.e., \(\phi \equiv \psi\) if and only if \(\models \phi \leftrightarrow \psi\)) is an equivalence relation.  We select one finitary representative from each equivalence class so that no two distinct formulas in \(\mathcal{L}\) are equivalent. Since \(A\) is finite, and we consider only finitary formulas, \(\mathcal{L}\) is finite.} formulas based on a propositional logic with a finite set $A$ of atomic formulas. The variable \(\phi\) is then a generic variable ranging over all well-formed formulas in \(\mathcal{L}\), i.e., any expression that can be constructed from the atomic elements in $A$ using the logical connectives $\neg, \land, \lor, \rightarrow, \leftrightarrow$.

\begin{definition}
\noindent
    An \textbf{agenda} is defined as a finite subset $X \subseteq \mathcal{L}$ such that $\forall  \ \phi \in \mathcal{L}  \colon \phi \in X \Rightarrow \neg \phi \in X$. An agenda $X$ is called \textbf{$\wedge$-stable}, if $\forall \ r \in \mathbb{N}, \ \forall \phi_{1}, ..., \phi_{r} \in X \colon \bigwedge_{j=1}^{r} \phi_{j} \in X$.
A \textbf{propositional probability} is a map $P \colon \mathcal{L} \rightarrow [0,1]$ such that the following holds:
\renewcommand{\labelitemi}{\textbullet}
\begin{itemize}
    \item $P(\phi) = 1$ if $\models \phi$ (i.e. $\phi$ is a tautology);
    \item If $\phi_{1}, \phi_{2} ... \in \mathcal{L}$ such that $\models \neg (\phi_{i} \land \phi_{j})$ for all $i \neq j \in \mathbb{N}$, then

    \[P(\bigvee_{i=1}^{\infty} \phi_{i}) = \sum_{i=1}^{\infty} P(\phi_{i}) . \]
    
\end{itemize}
\end{definition}  
\subsection{The framework of probability aggregation} \label{fpa}
Each individual $i \in N$ with $N = \{1,2,\dots,n\}$ holds a \textit{probabilistic judgement} which is considered as a \textit{degree of confidence} in the truth of the propositions in the agenda $X$ that ranges from 0 (maximal doubt) to 1 (maximal confidence). 

\begin{definition}
\noindent
    The probabilistic judgement $J$ is a function $J \colon X \rightarrow [0,1]$. The judgement $J$ is called \textbf{probabilistically rational} if there exists a propositional probability $P^{*} \colon \mathcal{L} \rightarrow [0,1]$ on the whole set $\mathcal{L}$ such that $\forall \ \phi \in X \colon J(\phi) = P^{*}(\phi)$.
    A \textbf{probabilistic profile} on an agenda $X$ is a tuple $(J_{1}, ..., J_{n})$, where each of $J_{1}, ..., J_{n}$ is a probabilistically rational probabilistic judgement. The \textbf{probabilistic aggregation rule} is a function $F \colon J(X)^{n} \rightarrow J(X), $
where $J(X)$ is the collective set of all probabilistically rational probabilistic judgements on the agenda $X$.
\end{definition}
Note that the propositional probability $P^{*}$ is an extension of the probabilistic judgement $J$. The probabilistic profile captures the landscape of beliefs across group members, while the aggregation rule consolidates these individual probabilities into a unified collective judgment that maintains probabilistic rationality.

\subsection{A Running Example: A Political Scenario}
\label{example}

To illustrate the concepts introduced in Section~\ref{fpa}, consider a $\wedge$-stable agenda $X$ that comprises the following atomic propositions along with their logical interconnections:
\begin{description}[leftmargin=1cm, labelindent=0.5cm, labelsep=0.5cm, itemsep=0.5ex]
    \item[$a$:] \textit{Population ageing is straining the healthcare and pension systems.}
    \item[$b$:] \textit{It is necessary to implement new immigration policies to attract more skilled professionals who can contribute to the economy.}
    \item[$c$:] \textit{The current tax revenue is sufficient to meet the fiscal demands imposed by population ageing.}
\end{description}
Table~\ref{tab:initial} shows a sample \textit{probabilistic profile} on the agenda $X$, where we have three individuals (denoted by $J_1$, $J_2$, and $J_3$). This example serves as our running example throughout the paper.

\begin{table}[ht]
\centering
\renewcommand{\arraystretch}{1.4}
\caption{A sample probabilistic profile on the agenda $X = \pm \{a, b, c, a \rightarrow b, a \land b, a \land c, b \land c, a \land b \land c \}$, where each individual assigns a degree of confidence (ranging from 0 to 1) to the propositions. Each entry is chosen so as to align with the relationships dictated by logical connectives, ensuring that bounds for probability assignments are satisfied. Note that since $X$ is $\wedge$-stable, this agenda uniquely 
  determines a single joint distribution over $\{a,b,c\}$.}
\begin{tabular}{lllllllll}
\toprule
Individual \ & $a\ \ \ \ $ \ & $b\ \ \ \ $ \ & $c\ \ \ \ $ \ & $a \rightarrow b\ $ \ & $a \land b\ $ \ & $a \land c\ $ \ & $b \land c\ $ \ & $a \land b \land c\ $ \\
\midrule
\(J_1\) & 0.7 & 0.8 & 0.4 & 0.9 & 0.6 & 0.3 & 0.3 & 0.25 \\
\(J_2\) & 0.7 & 0.5 & 0.4 & 0.8 & 0.5 & 0.3 & 0.25 & 0.15 \\
\(J_3\) & 0.7 & 0.4 & 0.4 & 0.6 & 0.3 & 0.3 & 0.15 & 0.1 \\
\bottomrule
\end{tabular}
\vspace{0.8em}
\label{tab:initial}
\end{table}

\section{Probability Aggregation} \label{probagr}

\subsection{Static rationality} \label{static}
Our first theorem is an adaptation of the characterization result from \cite{DietrichListGeneralizedAgendas} on the linear pooling function. In order to state the theorem in our framework of probability aggregation, we need to introduce the following constraints on the agenda $X$, formulas $\phi \in \mathcal{L}$ and the probabilistic aggregation rule $F$.

\renewcommand{\labelitemi}{\textbullet}
\begin{definition}
\noindent
    An agenda $X$ is called \textbf{nested}, if there exist formulas $\phi_{1}, ..., \phi_{r} \in X$ such that the following holds:
\begin{itemize}
    \item $\forall \ j=1,..., r-1 \colon \models \phi_{j} \rightarrow \phi_{j+1}$;
    \item $X= \bigcup_{j=1}^{r} \ (\{\phi_{j}\} \cup \{\neg \phi_{j}\})$.
\end{itemize}
Otherwise the agenda $X$ is called \textbf{non-nested}. A formula $\phi \in \mathcal{L}$ is called \textbf{contingent} if neither $\models \phi$ nor $\models \neg \phi$. We denote the set of contingent elements of $Y \subset \mathcal{L}$ by $cont(Y)$.
\end{definition} 
In the next step, we will show that it is possible to characterize linear probabilistic aggregation rules for a very large class of agendas -- all non-nested agendas. In the setting of probabilistic opinion pooling, the agenda has been traditionally assumed to be a $\sigma$-algebra (i.e.  closed under complementation and countable union of events). Dropping the assumption of  a $\sigma$-algebra, \cite{DietrichListGeneralizedAgendas} consider more general agendas. The constraint of nestedness implies that formulas $\phi_{1}, ..., \phi_{r} \in X$ are logically interconnected by material implication which corresponds to the subset-relations in a nested agenda $X$ in the framework of \cite{DietrichListGeneralizedAgendas}. Intuitively, non-nestedness means that there is no subagenda in $X$ that 
comprises only propositions that are logically interrelated by material implications and that are closed under negation. Thus, the propositions under consideration can be probabilistically dependent (correlated) without being logically interrelated by material implication.

The intuition of consensus preservation requirements is well-known in the literature on collective decision theory: a proposition in the agenda should have a collective probability of 1 if every group member assigns it a probability of 1 (certainty). In judgement aggregation theory the analogue of this requirement is referred to as \textit{unanimity preservation} criterion \cite{papermongin08}, and in probability aggregation theory it is often called \textit{zero probability property} \cite{papermcconway}. Here, we introduce a version of consensus preservation that holds even if group members' beliefs are not revealed in the process of decision-making, similarly to \cite{DietrichListGeneralizedAgendas}. 
 
 \begin{definition}\label{def4}
 \noindent
     We say a judgement is \textbf{consistent with truth of} $\phi^{*} \in \mathcal{L}$, if there exists a propositional probability $P^{*} \colon \mathcal{L} \rightarrow [0,1]$ such that the following holds: 
\begin{itemize}
    \item $\forall \ \phi \in X \colon P^{*}(\phi)=J(\phi)$;
    \item $P^{*}(\phi^{*})=1$,
\end{itemize}
where the propositional probability $P^{*}$ is an extension of the individual probabilistic judgement $J$. A probabilistic aggregation rule $F$ is called \textbf{consensus compatible} if the following holds: For all $\phi \in \mathcal{L}$ and for all $J_{1}, ..., J_{n} \in J(X)$, if $J_{1}, ..., J_{n}$ are consistent with truth of $\phi$, then also $F(J_{1}, ..., J_{n})$ is consistent with truth of $\phi$.
 \end{definition} 
 Intuitively, Definition~\ref{def4} implies that if there is a possibility that all group members assign the probability of 1 to a formula $\phi$ (though these beliefs may remain unrevealed), then the collective opinion must reflect this certainty in $\phi$.

 \begin{definition}
 \noindent
     A probabilistic aggregation rule $F$ is called \textbf{independent} if for every $\phi \in X$ there exists a function $S \colon [0,1]^{n} \rightarrow  [0,1]$ such that, for all $J_{1}, ..., J_{n} \in J(X)^{n}$, $F(J_{1}, ..., J_{n})(\phi) = S(J_{1}(\phi), ..., J_{n}(\phi)) .$
 \end{definition}
The independence requirement asserts that the aggregate probability of each formula $\phi$ in the agenda $X$ should depend only on the individual probability assignments to $\phi$. In the literature on probability aggregation, this condition has been referred to as \textit{weak setwise function property} \cite{paperaczelwagner}. 

\begin{definition}
\noindent
    A probabilistic aggregation rule $F$ is called \textbf{linear}, if for every profile of individual probabilistic judgements $(J_1,..., J_n)$ and every formula $\phi \in X$, the collective probability for the proposition $\phi$ is the weighted average of individual values of probabilistic judgements for $\phi$:

\[F(J_{1}(\phi),..., J_{n}(\phi))= w_{1}J_{1}(\phi) + w_{2}J_{2}(\phi) + ... + w_{n}J_{n}(\phi),\]
where $w_1, \dots , w_n$ are fixed non-negative individual weights with sum-total of 1.
\end{definition}
The individual weights may vary, for instance, a \textit{dictatorship} of individual $i$ would imply that $w_{i}=1$ and $w_{j}=0$ for all other individuals $j\in N$ \cite{paperopinion}. It can be shown that the following general characterization result for linear probabilistic aggregation rules holds in our framework.

\begin{theorem} \label{linear}
\noindent
Let $F \colon J(X)^{n} \rightarrow J(X)$ be consensus compatible and independent. Let $X$ be non-nested. Let $\vert cont(X) \vert > 4$. Then $F$ is linear.
\end{theorem}

\begin{proof}
Let $S = \{0,1\}^{|A|}$ be the set of all truth assignments to the finite set of atomic propositions \(A\). For every formula \(\phi \in \mathcal{L}\) define $[\phi] = \{ v \in S \mid v \text{ satisfies } \phi \}.$ It is easy to show that the mapping $I: \mathcal{L} \to 2^S, \ \phi \mapsto [\phi]$ is bijective. Since $I$ preserves the Boolean operations, i.e.,
\[
I(\phi\land\psi)=I(\phi)\cap I(\psi),\quad I(\phi\lor\psi)=I(\phi)\cup I(\psi),\quad I(\neg\phi)=S\setminus I(\phi),
\]
it is an isomorphism. For each truth assignment $v \in S$, let $\chi_v$ be a formula that is true only at $v$. Then define the set function $\mu: 2^S \to [0,1]$ on $2^S$ induced by the propositional probability $P: \mathcal{L} \to [0,1]$  via
\[
\mu(M) = P\Bigl( \bigvee_{v\in M} \chi_v \Bigr),
\]
for every $M \subseteq S$. Since $P$ satisfies the Kolmogorov axioms, \(\mu\) naturally defines a full probability measure on \(2^S\).

Since \(I\) is bijective, any non-nested agenda \(X\subseteq\mathcal{L}\) with \(|cont(X)|>4\) is mapped to a non-nested set system \(I(X)\) in \(2^S\) with \(I(cont(X)) = 2^S \setminus \{S,\emptyset\}\) and hence \(|I(cont(X))|>4\). Now let \(F:J(X)^n\to J(X)\) be a consensus compatible and independent aggregation rule. Define the induced operator
\[
\widetilde{F}:J(I(X))^n\to J(I(X))\quad\text{by}\quad \widetilde{F}(I(J_1),\dots,I(J_n))=I\Bigl(F(J_1,\dots,J_n)\Bigr).
\]
Since \(I\) preserves truth values, consensus compatibility of \(F\) implies that if \(J_i(\phi)=1\) for all \(i\) then \(\widetilde{F}(I(J_1),\dots,I(J_n))(I(\phi))=1\). Therefore, \(\widetilde{F}\) is consensus compatible in the sense of \cite{DietrichListGeneralizedAgendas}. By Theorem 4a in Sect. 5.1 in \cite{DietrichListGeneralizedAgendas}, any consensus-compatible and independent aggregation rule on a non-nested agenda with \(|I(cont(X))|>4\) is linear; hence, \(F\) is linear.\hfill $\square$
\end{proof}
\begin{table}[ht]
\centering
\renewcommand{\arraystretch}{1.4}
\caption{Collective probabilistic profile obtained via the linear aggregation rule with equal weights $\frac{1}{3}$ for each individual and values rounded to four decimals.}
\begin{tabular}{lllllllll}
\toprule
Individual \ & $a\ \ \ \ $ \ & $b\ \ \ \ $ \ & $c\ \ \ \ $ \ & $a \rightarrow b\ $ \ & $a \land b\ $ \ & $a \land c\ $ \ & $b \land c\ $ \ & $a \land b \land c\ $ \\
\midrule
\(J_1\) & 0.7 & 0.8 & 0.4 & 0.9 & 0.6 & 0.3 & 0.3 & 0.25 \\
\(J_2\) & 0.7 & 0.5 & 0.4 & 0.8 & 0.5 & 0.3 & 0.25 & 0.15 \\
\(J_3\) & 0.7 & 0.4 & 0.4 & 0.6 & 0.3 & 0.3 & 0.15 & 0.1 \\
\midrule
$F(J_1,J_2,J_3)\ $ & 0.7\ \ \ & 0.5667\ \ \ & 0.4\ \ \ & 0.7667\ \ \ & 0.4667\ \ \ & 0.3\ \ \ & 0.2333\ \ \  & 0.1667\  \\
\bottomrule
\end{tabular}
\vspace{0.8em}
\label{tab:aggregated}
\end{table}
Table~\ref{tab:aggregated} demonstrates the resulting collective judgement in the running example using the linear aggregation rule. Note that the conditions of Theorem~\ref{linear} (e.g. non-nestedness, number of contingent elements) are satisfied in our example.  Intuitively, Theorem~\ref{linear} does not exclude the possibility of non-dictatorial and probabilistically rational aggregation of individual attitudes in our framework. Non-dictatorship means there is no individual $i \in N$ such that, for every profile $(J_1,\ldots, J_n)$ in the domain of $F$ and every formula $\phi \in X$, the collective judgement for $\phi$ is given by $F(J_1(\phi),\ldots,J_n(\phi))= J_i(\phi)$.  Linearity encompasses both the case of dictatorship (if one weight equals 1 and the others 0) and the possibility of non‐dictatorial aggregation (if the weights are chosen otherwise). The possibility of non-dictatorial aggregation of individual attitudes on logically interrelated issues arises because the domain of the probabilistic aggregation rule $F$ is extended in our model allowing for probabilistic individual attitudes, as compared to impossibility results from the standard framework of judgement aggregation (see e.g. \cite{paper2002,paperdietrich07,papermongin08}).

\subsection{Dynamic Rationality} \label{dynamic}
Now, suppose the group learns the truth of one of the propositions in the agenda of interest. This raises the question: How should the group update their beliefs in response to this new information? A convincing criterion for the quality of this updating process is dynamic rationality—the requirement that the outcome of belief revision should be invariant to the order in which updating and aggregation occur \cite{dietrichlist24}. In other words, whether individuals first update their personal probabilistic judgements and then aggregate them, or first aggregate their personal probabilistic judgements and then update the collective judgement, the result should remain the same.

We define the updating process as follows. Initially, an individual holds a probabilistic judgement $J(\phi)$ with respect to formula $\phi$ as well as probability $J(\psi)$ for all other formulas $\psi \in X$. Later she learns the truth of $\phi$ and adopts a new probabilistic judgement $J^{\phi}(\psi)$. The updating process is formalized by an \textit{update operator} which is a mapping $U:J(X) \times X \to J(X)$.

In this paper, we are interested in scenarios where group members initially agree on the probability assignments for a subset of propositions within the agenda, denoted as $\Phi \subseteq X$. To formalize this setup, we introduce the notion of a \textit{common ground} within the agenda, referring to cases where individuals share identical probabilistic judgments on $\Phi$. The learning process of the group then consists of successively learning the truth or falsity of the propositions in $\Phi$ and the group members having to adjust their probabilistic judgements about the remaining propositions in $X$ accordingly. 

\begin{definition}
\noindent
    Assume $X$ to be some $\wedge$-stable agenda of interest, and let $\Phi \subseteq X$ be $\wedge$-stable. We define the domain $\mathcal{D}_{\Phi}$ induced by $\Phi$ by setting
\[\mathcal{D}_{\Phi} = \{(J_1, \dots, J_n) \in J(X)^n \ \vert \ \forall \phi \in \Phi \ \forall i, j \colon J_{i}(\phi)=J_{j}(\phi) \neq 0\} .\]
We call $\Phi$ the \textbf{common ground} of our aggregation problem. A \textbf{$\Phi$-aggregation operator} is a mapping $F:\mathcal{D}_{\Phi} \to J(X)$. We call an update operator $U$ \textbf{common ground preserving} (or $\Phi$-\textbf{preserving}), if for all $\phi,\psi \in \Phi$ and all $(J_1, \dots, J_n) \in \mathcal{D}_{\Phi}$ it holds that $$U(J_1,\phi)(\psi)= \dots = U(J_n,\phi)(\psi).$$
If $F$ is a $\Phi$-aggregation operator and if $U$ is a $\Phi$-preserving update operator, we call the pair $(F,U)$ \textbf{probabilistic dynamic rational w.r.t. $\Phi$}, if for all $\phi \in \Phi$ and $\psi \in \mathcal{L}$ it holds that
$$F(U(J_1,\phi), \dots , U(J_n,\phi))(\psi)=U(F(J_1,\dots , J_n),\phi)(\psi).$$
\end{definition}  
The requirement of probabilistic dynamic rationality is a plausible constraint, since it demands that the individual and collective probabilistic judgements are updated after learning the truth of $\phi$ in such a manner that it makes no difference whether they are updated before aggregation or after aggregation process. Thus, it intuitively corresponds to the requirement of \textit{external Bayesianity} from the literature on probability aggregation \cite{papergenest} and to the condition of dynamic rationality in the judgement aggregation framework of \cite{dietrichlist24}.
The following result illustrates that the possibility of dynamically rational aggregation with a $\Phi$-preserving update operator on the restricted domain $\mathcal{D}_{\Phi}$.

\begin{theorem} \label{dynrat}
\noindent
Assume $X$ to be some $\wedge$-stable agenda of interest, and let $\Phi \subseteq X$ be $\wedge$-stable. If $F:\mathcal{D}_{\Phi} \to J(X)$ is linear and $U$ is defined by $$U:J(X) \times X \to J(X)~~~,~~~(J,\phi) \mapsto J^{\phi},$$
where $J^{\phi}(\psi):= \frac{J(\phi \land \psi)}{J(\phi)}$ for all $\psi \in J(X)$ is the \textbf{Bayesian updated probabilistic judgement of $\psi$ given the truth of $\phi$},
then the pair $(F,U)$ is probabilistic dynamic rational w.r.t. $\Phi$ and $U$ is $\Phi$-preserving.
\end{theorem}

\begin{proof}

Assume that the group $N=\{1, 2, 3, ..., n\}$ learns the truth of formula $\phi \in \Phi$. 
Consider now the aggregation of updated individual judgements $$F(U(J_1,\phi), \dots , U(J_n,\phi))(\psi)= \frac{1}{n} \sum^{n}_{i=1} U(J_i,\phi)(\psi) = \frac{1}{n} \sum^{n}_{i=1} \frac{J_{i}(\phi \land \psi)}{J_{i}(\phi)}$$

and the updating of the aggregated collective judgement

$$U(F(J_1,\dots , J_n),\phi)(\psi)= \frac{F(J_1,..., J_n)(\phi \land \psi)}{F(J_1,..., J_n)(\phi)}=\frac{\frac{1}{n} \sum^{n}_{i=1} J_{i}(\phi \land \psi)}{\frac{1}{n} \sum^{n}_{i=1} J_{i}(\phi)}.$$

By the definition of the domain $\mathcal{D}_{\Phi}$, where $\forall i, j \in N \colon J_{i}(\phi)=J_{j}(\phi)$ for all $\phi \in \Phi$, it follows that:
$$U(F(J_1,\dots , J_n),\phi)(\psi)= \frac{1}{n} \sum^{n}_{i=1} \frac{J_{i}(\phi \land \psi)}{J_{i}(\phi)} = F(U(J_1,\phi), \dots , U(J_n,\phi))(\psi).$$

Hence, the pair $(F,U)$ is probabilistic dynamic rational w.r.t. $\Phi$. Since $\forall i, j \in N \colon J_{i}(\phi)=J_{j}(\phi)$ for all $\phi, \psi \in \Phi$, we have:

\[
U(J_i, \phi)(\psi) = \frac{J_i(\phi \land \psi)}{J_i(\phi)}
= \frac{J_j(\phi \land \psi)}{J_j(\phi)} = U(J_j, \phi)(\psi)
\]

Thus, the update operator $U$ is $\Phi$-preserving.
\hfill $\square$
\end{proof}
\begin{table}[ht]
\centering
\renewcommand{\arraystretch}{1.4}
\caption{Updated probabilistic judgements after learning \(a\), with values rounded to four decimals. The last row, denoted by \(F(\square)\), displays the collective judgement resulting from updating individual judgements and aggregating them via the linear rule \(F(U(J_1,a), U(J_2,a) , U(J_3,a))\).}
\begin{tabular}{lllllllll}
\toprule
Individual \ & $ a \ $ \ & $b\ \ \ \ \ $ \ & $c\ \ \ \ \ $ \ & $a \rightarrow b\ \ $ \ & $a \land b\ \ $ \ & $a \land c\ \ $ \ & $b \land c\ \ $ \ & $a \land b \land c\ \ $ \\
\midrule
\(J^{a}_1\) & 1\  & 0.8571\ \ \ \  & 0.4286\ \ \ \   & 0.8571\ \ \ \   & 0.8571\ \ \ \   & 0.4286\ \ \ \ & 0.3571\ \ \ \  & 0.3571\ \ \ \  \\
\(J^{a}_2\) & 1 & 0.7143 & 0.4286 & 0.7143 & 0.7143 & 0.4286 & 0.2143 & 0.2143 \\
\(J^{a}_3\) & 1 & 0.4286 & 0.4286 & 0.4286 & 0.4286 & 0.4286 & 0.1429 & 0.1429 \\
\midrule
\(F(\square)\) & 1 & 0.6667  & 0.4286 & 0.6667  & 0.6667 & 0.4286 & 0.2381 & 0.2381 \\
\bottomrule
\end{tabular}
\vspace{0.6em}
\label{tab:updated_a}
\end{table}
\begin{table}[ht]
\centering
\renewcommand{\arraystretch}{1.4}
\caption{Updated probabilistic judgments after learning \(\neg c\), following the update on \(a\), with values rounded to four decimals. The probabilities for \(b\) and \(a\rightarrow b\) are computed as \(J^{a,\neg c}(b)= J^{a,\neg c}(a \rightarrow b)=\frac{J^a(b\land \neg c)}{J^a(\neg c)}=\frac{J^a(b)-J^a(b\land c)}{J^a(\neg c)} \), and for \(a \land b\) as \(J^{a,\neg c}(a\land b)=\frac{J^{a}((a\land b)\land \neg c)}{J^a(\neg c)} =\frac{J^{a}(a\land b)-J^{a}(a\land b\land c)}{J^a(\neg c)}\), using the value \(J^a(\neg c)=0.5714\). Here, \(F(*)\) abbreviates \(F(U(J_1,a, \neg c), U(J_2,a, \neg c) , U(J_3,a, \neg c))\).} 
\begin{tabular}{lllllllll}
\toprule
Individual \ & $\ a\ \ $ \ & $b\ \ \ \ \ $ \ & $c\ \ \ $ \ & $a \rightarrow b\ \ $ \ & $a \land b\ \ $ \ & $a \land c\ \ $ \ & $b \land c\ \ $ \ & $a \land b \land c\ \ $ \\
\midrule
\(J^{a,\neg c}_1\) & 1\ \  \ \ \ \ & 0.8750\ \ \ \  \ \ & 0\ \  \ \  \ \ & 0.8750\ \ \ \  \ \ & 0.8750\ \ \ \  \ \ & 0\ \ \ \  \ \ & 0\ \ \ \  \ \ & 0\ \  \ \ \\
\(J^{a,\neg c}_2\) & 1 & 0.8750 & 0 & 0.8750 & 0.8750 & 0 & 0 & 0 \\
\(J^{a,\neg c}_3\) & 1 & 0.5000 & 0 & 0.5000 & 0.5000 & 0 & 0 & 0 \\
\midrule
\(F(*)\) & 1 & 0.7500 & 0 & 0.7500 & 0.7500 & 0 & 0 & 0 \\
\bottomrule
\end{tabular}
\vspace{0.6em}
\label{tab:updated_neg_c}
\end{table}
Let us illustrate the dynamically rational aggregation on our running example. Here, the common ground $\Phi$ comprises the propositions $a$, $c$, and their conjunction $a \land c$—that is, the issues on which all individuals initially share equal confidence in the initial Table~\ref{tab:initial}. The first learning iteration is shown in Table~\ref{tab:updated_a}, in which all group members learn the truth of $a$. As a consequence, the updated probability for $a$ is set to $J^{a}(a) = \frac{J(a\land a)}{J(a)} = 1$. The updated probability for each proposition $\psi \in X$ is computed as $J^{a}(\psi)=\frac{J(a \land \psi)}{J(a)}$, using the value $J(a) = 0.7$ from Table~\ref{tab:initial}. 

Note that when we update the individual first and then aggregate, we obtain the collective updated values shown in Table~\ref{tab:updated_a}. When we aggregate first and then update, we compute the collective probability for propositions \(b\) and \(a\rightarrow b\) with values from Table~\ref{tab:aggregated} as $\frac{F(J_1,J_2,J_3)(a\land b)}{F(J_1,J_2,J_3)(a)}=0.6667;$ for \(c\) and \(a\land c\) as $\frac{F(J_1,J_2,J_3)(a\land c)}{F(J_1,J_2,J_3)(a)}=0.4286;$ and for conjunctions \(a \land b\) and \(a \land b \land c\) as $\frac{F(J_1,J_2,J_3)(a \land b \land c)}{F(J_1,J_2,J_3)(a)}=0.2381.$ Since both ways to proceed yield identical results for all propositions in the agenda, it is apparent that regardless of whether we first aggregate and then update or update each individual judgement and then aggregate, the final collective judgement stays the same.

Table~\ref{tab:updated_neg_c} then illustrates the second learning iteration, in which, after having learned $a$, the individuals learn the falsehood of  $c $ (i.e. the truth of  $\neg c $). Importantly, while the update on propositions that involve c (such as  $c $,  $a \land c $,  $b \land c $, and  $a \land b \land c $) drops to 0, the common ground is preserved: the judgments for $a$, $c$, and $a \land c$ remain identical over all individuals even after successive updates.

This \textit{sequential decision-making process}—whereby new evidence is incorporated in stages while preserving the common ground—is a crucial aspect of our model. Not only does it demonstrate dynamic rationality (i.e. that the order of updating and aggregation does not affect the final collective judgment), but it also offers a promising approach to sequential collective decision-making, a topic that has received little attention in traditional judgment aggregation frameworks.

This perspective sheds new light on the significance of Theorem~\ref{dynrat}, which establishes that, on the restricted domain $\mathcal{D}_{\Phi}$, a linear aggregation rule achieves dynamic rationality for logically interconnected issues. Since our linear aggregation rule $F$ satisfies consensus compatibility and independence (as shown in Theorem~\ref{linear}), our result demonstrates that—under these two constraints—a positive possibility for dynamically rational aggregation is attainable. It is important to note that Theorem~\ref{dynrat} does not directly address the issues of dictatorship or oligarchy in the aggregation process. The collective outcome is determined by the individual weights assigned in the linear rule; hence, while the aggregation function could be dictatorial for certain weight choices, our result shows that non-dictatorial (and thus non-oligarchical) aggregation is possible with appropriate weight selection.

A key aspect of our model is the restricted domain $\mathcal{D}_{\Phi}$, which is justified by the notion of \textit{fair learning}. Here, learning is deemed fair when group members can only update on propositions in the common ground $\Phi \subseteq X$, where all members have equal degrees of confidence (i.e., $\forall \phi \in \Phi,\ \forall i,j \in N:J_i(\phi)=J_j(\phi)$). Under fair learning, any Bayesian update on $\phi\in\Phi$ produces the same posterior shift in every individual’s beliefs, thereby preserving dynamic rationality. Our epistemic notion of fairness is grounded in equal information processing, guaranteeing that no agent is advantaged or disadvantaged by newly received data.

It has been established in \cite{dietrichlist24} that under standard aggregation axioms (e.g. universal domain, systematicity) no judgment aggregation rule can be dynamically rational with respect to any revision operator satisfying basic conditions on revision, if the propositions in the agenda are non-trivially interrelated. By contrast, when we recast the same agenda in our probabilistic framework -- replacing “yes/no” verdicts with probability assignments over the logically interconnected propositions -- and impose a common-ground domain restriction -- weighted linear pooling both preserves collective coherence and commutes with Bayesian updating. 

Thus, our result points out that impossibility theorems in dynamically rational judgment aggregation such as \cite{dietrichlist24} arise when fair learning is absent -- namely, when members begin with differing degrees of confidence. Our model not only supports fair learning but also underscores its necessity for maintaining rationality in dynamic collective decision-making.

Furthermore, while previous research (e.g., \cite{papergenest}) has shown that geometric pooling functions satisfy external Bayesianity and unanimity preservation, and \cite{genestzidek} demonstrated that the linear opinion pool fails external Bayesianity unless it is dictatorial, our framework adopts a different perspective. In Theorem~\ref{dynrat} we show that within our propositional probability logic framework, linear aggregation can be dynamically rational under suitable domain restrictions. This result does not conflict with earlier findings but rather highlights that linear averaging, when combined with our approach to belief updating, yields a dynamically rational aggregation rule -- an insight that opens promising new avenues for research and applications in dynamic and sequential collective decision-making.

\section{Concluding Remarks} \label{crem}
Building on the implications of Theorem~\ref{dynrat}, we conclude that our framework generalizes classical judgment and preference aggregation approaches, offering new insights into dynamic and sequential collective decision-making. We demonstrated that under consensus compatibility and independence on a non-nested agenda, any aggregation rule must be linear. More importantly, when the domain is restricted to a common ground (i.e., under fair learning), dynamic rationality is achieved: Bayesian updating, whether performed before or after aggregation, yields the same collective judgment. A key contribution of our model is its emphasis on sequential decision-making, where common ground preservation across updates enables consistent multi-stage deliberation. Our result highlights the central role of fair learning in ensuring dynamic consistency and avoiding the impossibilities found in more general domains. While we do not directly address the normative choice of aggregation weights, our findings show that non-dictatorial and dynamically rational aggregation is possible under appropriate conditions.

Future work will explore relaxing the common ground assumption and extending our framework to more general classes of aggregation functions. We anticipate that our approach will not only enrich the theoretical foundations of collective decision-making but also inform practical applications in machine learning that require robust mechanisms for dynamically updating group beliefs.
\subsubsection{Acknowledgements} 

Polina Gordienko gratefully acknowledges the support of the Friedrich-Ebert-Stiftung Academic Foundation. We would like to thank three anonymous referees for their helpful comments.

\subsubsection{Disclosure of Interests} 

The authors have no competing interests to declare that are relevant to the content of this article.

%

%
%
%

\end{document}